\documentclass[11pt]{amsart}

\usepackage{
amsmath,
amssymb,
amscd,
amsfonts,
amsthm,
color,
enumitem,
epsfig,
float,
fullpage,
graphicx,
hyperref,
latexsym,
mathrsfs,
mathtools,
psfrag,
transparent,
url,
wasysym,
wrapfig,
tikz,
tikz-cd,
xr
}

\newtheorem{theorem}{Theorem}[section]
\newtheorem{definition}[theorem]{Definition}

\newtheorem{question}[theorem]{Question}
\newtheorem{conjecture}[theorem]{Conjecture}
\newtheorem{proposition}[theorem]{Proposition}
\newtheorem{lemma}[theorem]{Lemma}

\newtheorem{openquestion}[theorem]{Open Question}

\newtheorem{example}[theorem]{Example}
\newcommand\cC{\mathcal{C}}

\newcommand{\bR}{\mathbb{R}}
\newcommand{\bZ}{\mathbb{Z}}

\newcommand\ba{\mathbf{a}}
\newcommand\bb{\mathbf{b}}

\newcommand\bw{\mathbf{w}}
\newcommand\bx{\mathbf{x}}
\newcommand\by{\mathbf{y}}
\newcommand\bzero{\mathbf{0}}

\newcommand{\on}{\operatorname}

\newcommand{\codim}{\on{codim}}

\newcommand{\im}{\on{im}}

\newcommand{\Sym}{\on{Sym}}

\newcommand{\rk}{{\on{rk}\:}}

\newcommand{\Rot}{{\on{Rot}}}
\newcommand{\Hess}{{\on{Hess}}\:}

\newcommand{\ul}{\underline}

\newcommand{\sr}{\stackrel}

\newcommand{\wt}{\widetilde}

\def\lra{\longrightarrow}

\newcommand{\eps}{\epsilon}

\begin{document}

\title{How regularization affects the geometry of loss functions}

\author{
Nathaniel Bottman,
Y.\ Cooper,
and
Antonio Lerario
}

\maketitle

\begin{abstract}
What neural networks learn depends fundamentally on the geometry of the underlying loss function.
We study how different regularizers affect the geometry of this function.
One of the most basic geometric properties of a smooth function is whether it is Morse or not.
For nonlinear deep neural networks, the unregularized loss function $L$ is typically not Morse.
We consider several different regularizers, including weight decay, and study for which regularizers the regularized function $L_\eps$ becomes Morse.
\end{abstract}

\section{Introduction}
\label{s:intro}

When training artificial neural networks, one constructs a loss function $L(\alpha)$, then uses a gradient-based method to find a set of weights $\alpha_0$ for the network for which the value $L(\alpha_0)$ is small.
As we work to improve our understanding of artificial neural networks, studying the geometry of the loss function is important.

A standard technique in deep learning is to add a regularizer $R_\eps(\alpha)$ to the loss function $L(\alpha)$, producing a regularized loss function  
\begin{align}
L_\eps(\alpha)
\coloneqq
L(\alpha) + R_\eps(\alpha).
\end{align}

\noindent
We then train the neural network by applying a gradient-based method on this regularized loss function $L_\eps$.

In this paper, we make some of the first advances in understanding how the geometry of the loss function changes when different regularizers are added, in the setting of neural networks with nonlinear activation functions.
There are many common choices for the regularizer $R_\eps$, and we will discuss several.

To build a theoretically-supported understanding of when and how regularization helps, the fundamental starting place is to understand how the geometry of the loss function changes as different regularizers are added.
This paper provides a step in that direction.

\subsection{Morse functions}

One of the most important classes of smooth functions is the class of \emph{Morse functions}, i.e.\ those functions whose critical points satisfy a certain nondegeneracy condition.
Morse theory was originally developed in the mid-20\textsuperscript{th} century by Marston Morse, in order to study the topology of a manifold via the sublevel sets of a generic function on it.
In modern differential topology, Morse theory has become an indispensible and central tool.
We include in \S\ref{ss:morse_thy} a technical overview of the definition of the Morse property and two relevant results: that the Morse functions are generic among $\cC^\ell$ functions for every $\ell \geq 0$, and that there is a local normal form for Morse functions.

The topic of the current paper is to investigate when regularized loss functions are Morse.
Given an unknown function, it is helpful to establish whether a function is Morse, Morse--Bott, or neither.
Doing so provides important information about which mathematical tools can be brought to bear on studying the function.

\subsection{Prior work}

In \cite{cooper1}, Cooper showed that in several standard settings, the unregularized loss function $L$ of a deep neural network has a positive-dimensional locus of global minima.
In particular, their work showed that generically, $L$ is not a Morse function.
They showed furthermore in \cite{cooper2} that $L$ has other positive-dimensional critical loci, and that on some of those loci $L$ fails to even be Morse--Bott.

Now that it is understood that unregularized loss functions arising in deep learning are neither Morse nor Morse--Bott, a natural question that arises is whether regularized loss functions are Morse or Morse--Bott.
This is a natural and fundamental question, and in recent years multiple groups have studied the relationship between regularization and the Morse property.

In \cite{poggio}, Xu--Rangamani--Banburski--Liao--Galanti--Poggio introduce a new regularizer, with interest in whether the loss function becomes Morse with the addition of this regularizer.
In \cite{maxim2022morse}, Maxim--Rodriguez--Wang consider a high-level view on this question.
In the setting of complex algebraic geometry, they give a mathematical characterization of the relationship between the critical loci of unregularized and regularized functions.
Their setting is an important step, though not directly applicable to modern deep learning because they work over the complex numbers.

A different approach was taken by Mehta--Chen--Tang--Hauenstein.
They considered a more limited setting, but one more directly related to modern deep learning in \cite{tingting}, where they consider the relationship between regularization and the Morse property in the setting of deep \emph{linear} networks.
This is the case of a feedforward neural network with arbitrarily many layers, and with all of the activation functions equal to the identity.
In this setting, the authors first showed that the locus of global minima is positive-dimensional.
They then introduced a modification of the standard L2 regularizer, which they called the generalized L2 regularizer, and showed that with that choice of regularizer, the regularized loss function had isolated global minima.
In their main result they establish:

\medskip

\noindent
{\bf Theorem 1, \cite{tingting}.}
{\it
Let $L$ be the loss function of a deep linear network.
For almost all choices of $\vec\eps$, all dense critical points of $L+R_{\vec\eps}$ are isolated and nondegenerate.
}

\medskip

In the language we are using, Mehta--Chen--Tang--Hauenstein's work can be reinterpreted as showing that for feedforward neural networks with linear activation functions, the unregularized loss function is not Morse, but that the loss function does become Morse with the addition of the generalized L2 regularizer they introduce.

The works above are motivated by the relationship between regularization and the Morse property in a variety of settings.
For various technical reasons, it is more difficult to understand this relationship in the setting of deep nonlinear networks, and no results had yet been obtained in this setting most salient to modern deep learning.
In this work, we provide the first results in this direction for loss functions arising from deep nonlinear neural networks.

\subsection{Acknowledgments}

The authors thank Tomaso Poggio for inspiring this work and for helpful conversations.
This work was supported by ONR Grant N00014-21-1-2589.
A.L.\ is supported by the Alice and Knuth Wallenberg Foundation.

\section{Results}
\label{s:results}

In \S\ref{s:generalized_L2} of this paper, we prove Theorem \ref{thm:using_lerario}, which is the analogue of the main result of \cite{tingting} in the case of general artificial neural networks with arbitrary smooth activation function.
Specifically, we prove the following result.

\medskip

\noindent
{\bf Theorem \ref{thm:using_lerario}.}
{\it
Let $N$ be a neural network with a smooth activation function.
Let $L\colon\bR^d \to \bR$ be the loss function determined by $N$, the data set $D$, and a choice of a smooth loss (e.g.\ L2, cross entropy, etc.).

For a generic choice of $\vec{\eps} \in \bR^d$, the generalized-L2-regularized loss function
\begin{align}
L_{\vec\eps}
\coloneqq
L + R_{\vec{\eps}},
\qquad
R_{\vec{\eps}}(\alpha)
\coloneqq
\eps_1\alpha_1^2+\cdots+\eps_d\alpha_d^2
\end{align}
is Morse.
In particular, $L_{\vec\eps}$ has the property that its critical locus is a discrete subset of $\bR^d$.
}

\medskip

We go on to consider the standard L2 regularizer $\eps R_s$, and an additional regularizer $
\eps R_m$, and for each establish whether the regularized loss is Morse, Morse--Bott, or neither.

Interestingly, in \S\ref{s:standard_L2} we show that in the setting of feedforward neural networks with linear activation function considered by \cite{tingting}, the loss function does not become Morse with the addition of the standard L2 regularizer $\eps R_s$.
In other words, in their setting, their result does not extend from the generalized L2 regularizer to the standard L2 regularizer.
However, we conjecture that in the general case, for neural networks with nonlinear activation functions and enough data points, L2 regularized loss functions do become Morse.

In \S\ref{s:multiplicative_regularizer}, we consider a regularizer $\eps R_m$ introduced by Xu--Rangamani--Banburski--Liao--Galanti--Poggio in \cite{poggio}.
We find that for this regularizer, both for neural networks with linear and nonlinear activation functions, the loss function does not become Morse, or even Morse--Bott, when this regularizer is added.

Thus, we see that for some regularizers, like the generalized L2 regularizer, the loss function becomes Morse when this regularizer is added for all neural networks, with any architecture or activation function.
For others, like the multiplicative regularizer, the loss function does not become Morse when this regularizer is added, whether with linear or nonlinear activation.
But the most commonly used regularizer, the L2 regularizer, is the most subtle, and there remain open questions in this direction.
We know that for linear activation functions the loss function does not become Morse with the addition of this regularizer, we conjecture that for general feedforward neural networks with nonlinear activation functions the loss function does become Morse with the addition of this regularizer, and there remain interesting open geometric questions to study in this direction.

Finally, in the appendix, we prove two results about the relationship between $L2$-regularization and the Morse property in two different function classes.
The first, Theorem \ref{thm:answer_to_rho_q}, relies on the Jet Transversality Theorem and Sard's Theorem.
The second, Proposition \ref{prop:polys_L2-reg}, uses tools from semialgebraic geometry, and applies to the space of polynomials.
These results are not immediately applicable to deep learning --- both apply to classes of functions that are different than the class of functions given by loss functions of typical neural networks --- but they illustrate a set of tools that we expect will be broadly useful in further investigations.

\section{Basics}
\label{s:basics}

In this paper, we work in the setting of smooth functions, also called $\cC^\infty$ functions.
Recall that a function $\varphi: \bR^d \rightarrow \bR$ is \emph{smooth} if its derivatives $\frac{\partial^\ell\varphi}{\partial \alpha_{i(1)}\cdots\partial \alpha_{i(\ell)}}$ are defined for all $\ell$.
Because we work only with smooth functions, our results do not apply to neural networks with the commonly used activation function ReLU.
Our results do apply to neural networks using many other activation functions, including softplus, tanh, and sigmoid, as well as smoothed variants of ReLU.

\subsection{Background on Morse theory}
\label{ss:morse_thy}

In this subsection, we describe the \emph{Morse functions}, which are generic among the $\cC^\infty$ functions and should be thought of as exhibiting ``typical behavior''.

\begin{definition}[\S2, \cite{milnor2016morse}]
Suppose that $\varphi\colon \bR^d \to \bR$ is $\cC^\infty$.
\begin{itemize}
\item
A point $p \in \bR^d$ is a \emph{critical point of $\varphi$} if $\nabla\varphi(p) = \bzero$.

\smallskip

\item
The \emph{Hessian of $\varphi$ at $p$} is the symmetric $d\times d$ matrix
\begin{align}
\Hess\varphi(p)
\coloneqq
\left(
\frac
{\partial^2\varphi}{\partial\alpha_i\partial\alpha_j}
\right)_{1\leq i,j \leq d}.
\end{align}

\smallskip

\item
Suppose that $p$ is a critical point of $\varphi$.
We say that $p$ is \emph{degenerate} resp.\ \emph{nondegenerate} if $\Hess\varphi(p)$ is singular resp.\ nonsingular.

\smallskip

\item
We say that $\varphi$ is \emph{Morse} if each of its critical points is nondegenerate.
\end{itemize}
\end{definition}

Next, we state two properties of Morse functions that show how useful the notion of being Morse is.
The first result says that Morse functions are generic; the second says that Morse functions have a particularly simple local normal form.

\begin{proposition}[Weaker version of Corollary 6.8, \cite{milnor2016morse}]
Any bounded smooth function $\varphi\colon\bR^d \to \bR$ can be uniformly approximated by a Morse function $\psi$.
Moreover, for any $\ell \in \bZ_{\geq0}$ and compact set $K \subset M$, $\psi$ can be chosen so that $\psi|_K$ is $\cC^\ell$-close to $\varphi|_K$.
\end{proposition}

\begin{lemma}[Morse Lemma; paraphrase of Lemma 2.2, \cite{milnor2016morse}]
\label{lem:morse_lemma}
Suppose that $p$ is a nondegenerate critical point of a $\cC^\infty$ function $\varphi\colon \bR^d \to \bR$.
Then there is an integer $\ell \in [0,d]$ and local coordinates $\wt\alpha_1, \ldots, \wt\alpha_d$ on a neighborhood $U$ of $p$ such that the following formula holds on $U$:
\begin{align}
\varphi\bigl(\wt\alpha_1,\ldots,\wt\alpha_d\bigr)
=
f(p) - \wt\alpha_1^2 - \cdots - \wt\alpha_\ell^2
+ \wt\alpha_{\ell+1}^2 + \cdots + \wt\alpha_d^2.
\end{align}
\end{lemma}

\noindent
The integer $\ell$ is called the \emph{(Morse) index} of $p$.
An important consequence of Lemma \ref{lem:morse_lemma} is that nondegenerate critical points are isolated within the critical locus.
In particular, Morse functions have discrete critical loci.

\subsection{Notation}

We consider neural networks whose activation functions are all smooth.
In some sections we assume the neural network to be a fully connected feedforward network, of arbitrary width and depth, and in others we allow more general architectures.

The main object of study in this work will be the loss function $L$ associated to a neural network, which is the function on which gradient descent takes place.
We begin with the definition of a feedforward neural network, then consider more general neural networks.

\subsubsection{Feedforward neural networks}
We begin with the definition of a feedforward neural network.
First, fix a {\bf feedforward graph}, which is a directed graph that is stratified into {\bf layers} of widths $a, k_1, \ldots, k_\ell, b$, ordered from ``earliest'' to ``latest''.
In a feedforward graph, each edge is between nodes in adjacent layers, and they always go from an earlier layer to a later layer.
We call a feedforward network {\bf fully connected} if every possible edge is present.
(That is, between the $i$-th and $(i+1)$-st layers, there are $k_ik_{i+1}$ edges.)
To specify a neural network, we also choose an {\bf activation function}
\begin{align}
\sigma
\colon
\bR \to \bR.
\end{align}

Given a specific choice of weights and biases for the network, we construct a function as follows.
Given a vector \begin{align}
\alpha
=
(\bw,\bb)
\in
\bR^d
\end{align}
in the parameter space, we define the function 

\begin{align}
f_\alpha
=
f_{\bw,\bb}
\colon
\bR^a \to \bR^b
\end{align}
by composing the following sequence of maps specified by the neural network and the choice of $\bw, \bb$:

\begin{align}
\bR^a
\sr{M_1\bx+\bb_1}{\lra}
\bR^{k_1}
\sr{\sigma}{\lra}
\bR^{k_1}
\sr{M_2\bx+\bb_2}{\lra}
\cdots
\sr{\sigma}{\lra}
\bR^{k_{\ell-1}}
\sr{M_\ell\bx+\bb_\ell}{\lra}
\bR^{k_\ell}
\sr{\sigma}{\lra}
\bR^{k_\ell}
\sr{M_{\ell+1}\bx+\bb_{\ell+1}}{\lra}
\bR^b.
\end{align}

We comment on two aspects of this construction:
\begin{itemize}
\item
The components of $\bw$ form the entries of the $M_i$'s, and the components of $\bb$ form the entries of the $\bb_i$'s.
If the feedforward graph is not fully connected, then some of the entries of $M_i$ are forced to be zero, according to the rule that the $jj'$-th entry of $M_{i+1}$ can be nonzero only if there is an edge from the $j$-th node of the $i$-th layer to the $j'$-th node of the $(i+1)$-st layer.

\smallskip

\item
The arrow $\bR^{k_i}\sr{\sigma}{\lra}\bR^{k_{i}}$ indicates that we apply $\sigma$ componentwise.
\end{itemize}

\begin{example}
Consider a fully connected feedforward graph with layers of widths 1, 3, 1 and $\sigma = (u \mapsto u^2+1)$.
The corresponding function space consists of those functions of the form
\begin{align}
f_\alpha
\colon
x
\mapsto
w_4\bigl((w_1x+b_1)^2+1\bigr)
+
w_5\bigl((w_2x+b_2)^2+1\bigr)
+
w_6\bigl((w_3x+b_3)^2+1\bigr)
+
b_4.
\end{align}
\null\hfill$\triangle$
\end{example}

\subsubsection{The loss function $L$}

In this subsubsection, we define the loss function
\begin{align}
L
\colon
\bR^d
\to
\bR.
\end{align}

\noindent
Fix
\begin{itemize}
\item
a feedforward neural network,

\smallskip

\item
a choice of activation function $\sigma$,

\smallskip

\item
a choice of loss $\zeta$,

\smallskip

\item
and a data set
\begin{align}
D
\coloneqq
\bigl\{(\bx_i,\by_i)\bigr\}_{i \in \{1,\ldots,n\}}
\subset
\bR^a\times\bR^b.
\end{align}
\end{itemize}

\noindent
Given this data, we define $L$ by:
\begin{align}
L(\alpha)
\coloneqq
\sum_{i=1}^n \zeta\bigl(f_\alpha(\bx_i) - \by_i\bigr).
\end{align}

\subsubsection{Additional classes of artificial neural networks}

Feedforward neural networks are an important class of artificial neural networks, both for theorists and practitioners.
Many other classes of neural networks are also used, including convolutional networks, graph neural networks, LSTMs, RNNs, and Transformers.
In all cases, a directed graph is the starting place, and the construction of the neural network gives a procedure for constructing a loss function $L: 
\bR^d 
\rightarrow \bR$.
\cite{herberg_notes} provides a nice overview of several important architectures.
The loss function $L$ is the central player in this work, and for each of our results we state the assumptions we need to make on how $L$ is constructed.

\subsection{Regularization}

We will consider regularizations of the loss function.
This means adding a small function $R_\eps$ to $L$:
\begin{align}
L_\eps(\alpha)
&\coloneqq
L(\alpha)
+
R_\eps(\alpha).
\nonumber
\end{align}
$L_\eps(\alpha)$ is called a regularized loss function.

\subsection{Regularizers}

In this paper, we consider the following regularizers.
\begin{itemize}
\item
The generalized L2 regularizer
\begin{align}
R_{\vec\eps}(\alpha)
\coloneqq
\eps_1 \alpha_1^2 + \cdots + \eps_d \alpha_d^2.
\end{align}
This regularizer was introduced by Mehta--Chen--Tang--Hauenstein in \cite{tingting}.

\medskip

\item
The standard L2 regularizer
\begin{align}
\label{eq:L2_regularizer}
\eps R_s(\alpha)
\coloneqq
\eps (\alpha_1^2 + \cdots + \alpha_d^2).
\end{align}

\medskip

\item
A multiplicative regularizer
\begin{align}
\eps R_m(\alpha)
\coloneqq \epsilon\|M_1\|^2_2\cdots\|M_{
\ell+1}\|^2_2.
\end{align}
This regularizer was introduced by Xu--Rangamani--Banburski--Liao--Galanti--Poggio in \cite{poggio}.
\end{itemize}

\section{The generalized L2 regularizer}
\label{s:generalized_L2}

We begin by considering the generalized L2 regularizer.
This regularizer was introduced by Mehta--Chen--Tang--Hauenstein in \cite{tingting}.
They studied the case of linear neural networks, and in that setting established the following result.

\medskip

\noindent
{\bf Theorem 1, \cite{tingting}.}
{\it
Let $L$ be the loss function of a deep linear network.
For almost all choices of $\vec\eps$, all dense critical points of $L+R_{\vec\eps}$ are isolated and nondegenerate.
}

\medskip

\noindent
In this theorem, a choice of weights $w$ is called \emph{dense} if none of the entries are zero.
In our terminology, the conclusion of this result is that the restriction of the loss function to the complement of the coordinate hyperplanes is Morse.

In this section, we prove that their result holds in the more general setting of neural networks with any architecture, with arbitrary smooth activation function.
In this case, we show how to use a theorem of Lerario \cite{lerario} in order to prove that while the unregularized function $L$ is typically not Morse and has a positive-dimensional locus of global minima, for a generic choice of the generalized L2 regularizer the regularized function $L_\eps = L + \eps R_s$ is Morse, and in particular has isolated critical points.
More precisely, we prove the following result.

\begin{theorem}
\label{thm:using_lerario}
Let $N$ be a neural network with a smooth activation function.
Let $D$ be a data set.
Let $L\colon\bR^d \to \bR$ be the loss function determined by $N$ and $D$ and a choice of a smooth loss  $\zeta$ (e.g.\ L2, cross entropy, etc.).

For a generic choice of $\vec{\eps} \in \bR^d$, the generalized-L2-regularized loss function
\begin{align}
L_{\vec\eps}
\coloneqq
L + R_{\vec{\eps}},
\qquad
R_{\vec{\eps}}(\alpha)
\coloneqq
\eps_1\alpha_1^2+\cdots+\eps_d\alpha_d^2
\end{align}
is Morse, and in particular, $L_{\vec\eps}$ has the property that its critical locus is a discrete subset of $\bR^d$.
\end{theorem}

The key input is the following result of Lerario.

\begin{theorem}[Theorem 1.1, \cite{lerario}]
\label{thm:lerario}
Let $L$ be a smooth function on $\bR^d$ and $M \subset \bR^d$ a submanifold.
For $\vec{\eps} = (\epsilon_1,\ldots,\epsilon_d) \in \bR^d$, set $q_{\vec{\eps}} \colon \bR^d \to \bR$ to be the function defined by
\begin{align}
q_{\vec{\eps}}(x_1,\ldots,x_d)
\coloneqq
\eps_1 x_1^2+\cdots+\eps_d x_d^2.
\end{align}
Then the set
\begin{align}
A(L,M)
\coloneqq
\bigl\{
\vec\eps \in \bR^d
\:\big|\:
(L+q_{\vec{\eps}})|_M
\text{ is Morse}
\bigr\}
\end{align}
is residual in $\bR^d$.
\end{theorem}

\begin{proof}[Proof of Theorem \ref{thm:using_lerario}]
$L$ is a composition of smooth functions, hence is itself smooth.
The generalized L2 regularizer $R_{\vec{\eps}}$ is the same as the perturbation $q_{\vec{\eps}}$ used by Lerario.
Setting $M\coloneqq \bR^d$, we apply Theorem \ref{thm:lerario} to find that the regularized function $L_{\vec\eps}$ is a Morse function.

By the Morse Lemma (see Lemma \ref{lem:morse_lemma}), the critical points of $L_{\vec\eps}$ are discrete.
\end{proof}

\section{The standard L2 regularizer}
\label{s:standard_L2}

Having studied the generalized L2 regularizer, it is natural to ask whether the same results hold for the most commonly used regularizer, called the L2 regularizer, or weight decay.
The behavior when adding this regularizer is subtle, and we benefit from what we have learned in \S\ref{s:generalized_L2}.

In \S\ref{ss:linear_nets_not_morsifiable} we consider L2 regularization in the setting of neural networks.
We begin by considering L2 regularization in the more general setting of all smooth functions.

\subsection{L2 regularization of radially symmetric functions}

Consider the following question.

\medskip

\noindent
{\bf Question.}
Is it the case that for any smooth function $h$, the regularized function $h + \epsilon R_s$ is Morse for generic $\epsilon$?
(Here $R_s$ is defined in \eqref{eq:L2_regularizer}.)

\medskip

\noindent
The following example shows that the answer to this question is no.

\begin{example}
\label{ex:counterex}
Define a $\cC^\infty$ function $h\colon \bR^2 \to \bR$ by
\begin{align}
f(x,y)
\coloneqq
-\tfrac14(x^2+y^2)^2.
\end{align}
For $\eps \in \bR$, consider the $L2$-perturbation
\begin{align}
f_\eps(x,y)
\coloneqq
-\tfrac14(x^2+y^2)^2 + \tfrac12\eps(x^2+y^2).
\end{align}
Written in polar coordinates, this becomes
\begin{align}
\wt f_\eps(r,\theta)
=
-\tfrac14r^4 + \tfrac12\eps r^2.
\end{align}
We compute the gradient of $\wt f_\eps$:
\begin{align}
\nabla\wt f_\eps(r,\theta)
=
\langle
-r(r^2 - \eps),
0
\rangle.
\end{align}
It follows that for $\eps > 0$, the critical locus of $f_\eps$ is the union of the origin and the circle $\{r = \eps^{1/2}\}$.
In particular, for every $\eps > 0$, $f_\eps$ is not Morse.
\null\hfill$\triangle$
\end{example}

We are therefore led to the following question.

\begin{openquestion}
\label{openquestion}
Let $Q_s$ denote the subset of smooth functions $h: \bR^d \to \bR$ such that for all but finitely many $\eps$, $h + \eps R_s$ is Morse.
How do we characterize the set $Q_s$?

\end{openquestion}

\noindent

Let $\Rot$ denote the set of functions $h$ satisfying the following properties:
\begin{itemize}
\item
There exists a 1-parameter subgroup $G$ of $SO(d)$, a neighborhood $U$ of the identity in $G$, and an open set $V \subset \bR^d$ such that for every $\sigma\in U$ and $\alpha \in V$, we have the equality
\begin{align}
h(\alpha)
=
h(\sigma\cdot\alpha).
\end{align}

\medskip

\item
There exists a critical point $p$ of $h$ with the properties that $p$ lies in $V$, and $p$ is not fixed under the action of $G$.
\end{itemize}

\begin{proposition}
$Q_s \subset \Rot^c$.
\end{proposition}

\begin{proof}
We will prove the equivalent statement that if $h$ is an element of $\Rot$, then it does not lie in $Q_s$.
In fact, we will show that $h$ lying in $\Rot$ implies the stronger conclusion that for \emph{every} $\eps$, $h + \eps R_s$ fails to be Morse.

Let $p$ be the critical point of $h$ discussed in the second bullet of the definition of $\Rot$, and denote by $\gamma$ the intersection of $V$ with the orbit of $p$ under $U \subset G$.
By hypothesis, $\gamma$ is a 1-dimensional submanifold of $\bR^d$.
The invariance property of $h$ implies that $\gamma$ is contained in the critical locus of $h$.
By the Morse Lemma (c.f.\ Lemma \ref{lem:morse_lemma}), $h$ cannot be Morse.
\end{proof}

\noindent
Thus the question about the Morse property of the regularized function leads us to be interested in the rotational symmetry of the unregularized function.

\subsection{Linear networks do not become Morse under L2 regularization}
\label{ss:linear_nets_not_morsifiable}

In this subsection, we prove the following.

\begin{theorem}
\label{thm:linear_pos_dim_crit_locus}
Consider a linear fully connected feedforward neural network with all hidden layers of width at least 2 and also at least the width of the input and output layers.
We take the L2 loss $\zeta(x) = x^2$.
Assume that the best linear fits to the data points are not constant.
Let $L$ be the loss function defined by this network, and consider the standard L2 regularizer $\eps R_s = \eps \sum \alpha_i^2$.
The regularized function  $L_\eps = L + \eps R_s$ has a positive-dimensional locus of global minima, and in particular is not Morse, for any $\eps$ small enough.
\end{theorem}

Before we prove this theorem, we recall that in the case of linear networks, \cite{zhao2022symmetries} wrote down a set of symmetries for any fully connected feedforward network.
The symmetries are constructed in the following way: consider any layer, say the $i^{\text{th}}$, of the neural network.
The vector space corresponding to that layer is $\bR^{k_i}$, which is the target of the linear map $M_{k_{i-1}}$ and the domain of the linear map $M_{k_i}$.
The bias vector $\bb_{k_i}$ lies in this vector space $\bR^{k_i}$.

Zhao--Ganev--Walters--Yu--Dehmamy note that given any matrix $S \in GL(k_i)$, $S$ induces a map $\tilde{S}$ on the space of all parameters $\bR^d$ in the following way: 
\begin{align}
M_{k_{i-1}}
\mapsto
M_{k_{i-1}}S^{-1},
\qquad
M_{k_i}
\mapsto
SM_{k_i},
\qquad
\bb_{k_i}
\mapsto
S\bb_{k_i}.
\end{align}
The other parameters remain unchanged.

This provides a set of rotational symmetries for linear networks.
Given a layer of the neural network, say the $i^{\text{th}}$ layer, consider a rotation matrix $S \in SO(k_i)$.
By the construction above, this induces a map $\wt S\colon \bR^{d} \to \bR^{d}$.
This map $\wt S$ is a rotation.
In fact, because of the crucial fact that the activation function is the identity, $L$ is invariant under the action of $\wt S$.
We collect these facts in the following proposition.

\begin{proposition}
\label{prop:rotational_inv_of_L_linear_net}
The map $\wt S$ is an element of $SO(d)$.
Furthermore, $L$ is invariant under the action of $\wt S$.
\end{proposition}

\begin{proof}
The fact that $\wt S$ is a rotation is a consequence of the following three facts about rotation matrices:
\begin{itemize}
\item
The identity is a rotation.

\item
The inverse of a rotation is a rotation.

\item
If $A_1, \ldots, A_\sigma$ are rotations, then the block diagonal matrix $\text{Diag}(A_1,\ldots,A_\sigma)$ is a rotation.
\end{itemize}
The invariance of $L$ under the action of $\wt S$ follows from the fact that $\sigma$ is the identity: the $S$'s and $S^{-1}$'s cancel out.
\end{proof}

We can now prove Theorem \ref{thm:linear_pos_dim_crit_locus}.

\begin{proof}[Proof of Theorem \ref{thm:linear_pos_dim_crit_locus}]
{\bf Step 1:}
{\it
Define $\Xi$ to be the locus where $M_{k_{i-1}}$, $M_{k_i}$, and $\bb_{k_i}$ are all zero, and define $\Omega$ to be the set of global minimizers of $L$.
Then $\Omega$ is nonempty, and $\Omega \cap \Xi = \emptyset$.
}

\medskip

\noindent
Since we have taken $\sigma$ to be the identity, each $f_\alpha$ is a linear function.
We assumed the data set is not linearly best fit by any constant function.
Consider the set of solutions of linear regression for the data set.
Depending on the data set, there can be a unique solution, or a positive-dimensional family of solutions.

In the space of parameters $\bR^d$, $\Omega$ is the set of all parameters which map to the linear function(s) of best fit.
In particular, $\Omega$ is nonempty.
On the other hand, for $\alpha \in \Xi$, $f_{\alpha}$ is constant.
It follows from our hypothesis that $\alpha$ cannot be a global minimum of $L$.
Therefore $\Omega$ and $\Xi$ have empty intersection.

\medskip

\noindent
{\bf Step 2:}
{\it Define $\Omega_\eps$ to be the set of global minimizers of $L_\eps$.
Then $\Omega_\eps$ is nonempty, and $\Omega_\eps \cap \Xi = \emptyset$.}

\medskip

\noindent

The function $L_\eps$ is nonnegative, because both of its summands are.
Moreover, $\lim_{\alpha\to\infty} L_\eps(\alpha) = \infty$, because $L$ is nonnegative and $\eps R_s$ satisfies this same property.
It follows that there exists a global minimizer $\gamma$ of $L_\eps$.
This global minimizer is necessarily a critical point of $L_\eps$.
Hence $\Omega_\eps$ is nonempty.

Choose $\beta \in \Omega$.
For $\eps$ small enough, we have $L_\eps(\beta) < L_\eps(\alpha)$ for all $\alpha \in \Xi$.
It follows that if $\gamma$ is a global minimizer of $L_\eps$, $\gamma$ does not have $M_{k_{i-1}}$, $M_{k_i}$, and $\bb_{k_i}$ all zero.
Hence $\Omega_\eps$ does not intersect $\Xi$.

\medskip

\noindent
{\bf Step 3:}
{\it $L_\eps$ contains positive-dimensional critical loci, and in particular is not Morse.}

\medskip

\noindent
$L_\eps = L + \eps R_s$ is $SO(k_i)$-invariant, because both of its summands are.
It follows that the critical locus of $L_\eps$ is $SO(k_i)$-invariant.

All isolated fixed points of the $SO(k_i)$-action are contained in $\Xi$.
$\Omega_\eps$ is nonempty and has empty intersection with $\Xi$.
Take a critical point $\gamma \in \Omega_\eps$ and consider the $SO(k_i)$-orbit of $\gamma$.
This is a positive-dimensional locus and is contained in $\Omega_\eps$.
Therefore $L_\eps$ is not Morse.
\end{proof}

\subsection{General expectations}

\begin{conjecture}
\label{conj:L_not_in_Rot}
If the activation function $\sigma$ is a smooth function that is not linear, and the data set contains $\geq 2$ data points, then the loss function $L$ is not an element of $\Rot$.
\end{conjecture}

\noindent
Our roadmap to answering our motivating question --- whether loss functions arising in deep learning become Morse when regularized with weight decay --- is as follows.
First, study Open Question \ref{openquestion}, characterizing $Q$.
Second, determine whether the loss function $L$ is in $Q$.
We anticipate that a proof of Conjecture \ref{conj:L_not_in_Rot} will help to determine whether $L$ is in $Q$.

The proof of this conjecture and the resolution of Open Question \ref{openquestion} would lead to the answer to our motivating question --- whether loss functions arising in deep learning become Morse when regularized with weight decay.

\section{A multiplicative regularizer}
\label{s:multiplicative_regularizer}

Lastly, we study the effects of adding the multiplicative regularizer introduced by Xu--Rangamani--Banburski--Liao--Galanti--Poggio in \cite[equation (1)]{poggio}:
\begin{align}
\eps R_m(\alpha)
\coloneqq
\epsilon\|M_1\|^2_2\cdots\|M_{\ell+1}\|^2_2.
\end{align}

\noindent
In this section, we give a negative result.

\begin{proposition}
\label{prop:L_eps_has_pos_dim_crit_locus}
Let $L$ denote the L2 loss function for a fully connected feedforward neural network.
For this choice of regularizer, if the depth $\ell$ of the neural network is greater than or equal to 4, not only does 
$L_\eps = L + \eps R_m$
fail to be Morse, but it has a positive-dimensional locus of degenerate critical points.
\end{proposition}

\noindent
That is, $L_\eps$ fails to be Morse in two ways.
First, it has a positive-dimensional locus of critical points.
Second, it has degenerate critical points where all derivatives vanish to order more than 2 --- which means it is not Morse--Bott either.

We will prove this proposition by considering the core locus, defined by Cooper in \cite{cooper2}.

\begin{definition}[\S3.1.1, \cite{cooper2}]
For any integers $k_1, k_2$ satisfying $1 \leq k_1 < k_2 \leq \ell$, define the locus $C_{k_1,k_2}$ as follows:
\begin{align}
C_{k_1,k_2}
\coloneqq
\Bigl\{
(M_1,\bb_1,\ldots,M_{\ell+1},\bb_{\ell+1})
\:\Big|\:
M_{k_1}=M_{k_2}=M_{\ell+1}=0,
\:
\bb_{k_1}=\bb_{k_2}=0,
\:
\bb_{\ell+1} = \sum_{i=1}^n \by_i
\Bigr\}.
\end{align}
The \emph{core locus} $C$ is the union over all such loci:
\begin{align}
C
\coloneqq
\bigcup_{1 \leq k_1 < k_2 \leq \ell}
C_{k_1,k_2}.
\end{align}
\end{definition}

\noindent
In \cite{cooper2}, Cooper proved that the core locus is positive-dimensional, and that all derivatives vanish to order greater than 2 at every point in the core locus.
Now we will show the same holds for the function $L_\eps$, which implies Proposition \ref{prop:L_eps_has_pos_dim_crit_locus}.

\medskip

\begin{proof}[Proof of Proposition \ref{prop:L_eps_has_pos_dim_crit_locus}]
Let $p$ be any point in the core locus $C$.
We consider the derivatives of $L_\eps$ at $p$.
Because taking the derivative is a linear operation, any derivative of $L_\eps$ is the derivative of $L$ plus the derivative of $\eps R_m$.

All first derivatives of $L$ vanish at $p$ by \cite{cooper2}.
Now we consider $R_m$.
By direct computation, all first derivatives of $R_m$ vanish at $p$.
Hence the gradient of $L_\eps$ vanishes at every point $p$ in the core locus.
\end{proof}

\begin{proposition}
For depth $\ell \geq 2$ the origin is a degenerate critical point of $L_\eps$ with all derivatives vanishing to order more than 2.
\end{proposition}

\medskip

\begin{proof}
We consider the derivatives of $L_\eps$ at the origin.
Similarly to the prior proof, all first and second derivatives of $L$ vanish at the origin by \cite{cooper2}.
By direct computation, all first and second derivatives of $R_m$ vanish at the origin.
Hence all first and second derivatives of $L_\eps$ vanish at the origin.
\end{proof}

\noindent
So even for the shallowest neural networks, $L_\eps$ fails to be Morse, in the case of the multiplicative regularizer $R_m$.

\section{Conclusion}

In this paper, we have seen that the effect of adding a regularizer to the loss function $L$ varies considerably depending on which regularizer is added.
For the generalized L2 regularizer, the loss function becomes Morse when this regularizer is added for all neural networks, with any architecture or activation function.
Meanwhile, for the multiplicative regularizer, the loss function does not become Morse when this regularizer is added, either for linear or nonlinear activations.

The way adding a regularizer changes the geometry of the loss function $L$ can also be subtle.
For the most commonly used regularizer, the L2 regularizer, in the case of linear networks the loss function does not become Morse with the addition of this regularizer.
On the other hand, we conjecture that for general feedforward neural networks with nonlinear activation functions the loss function does become Morse with the addition of this regularizer.
Certainly geometric questions in the case of the L2 regularizer are more delicate than the other regularizers we study and there remain several interesting open geometric questions to study in this direction.

\subsubsection*{Broader Impacts}

This work focuses on the mathematical understanding of a technical practice in machine learning.
While this may feel removed from the machine learning systems that are beginning to be integrated into our daily lives, over time theoretical advances will improve the performance of machine learning systems.
Therefore this work will have broad societal impacts.

As the authors of this work, we have a responsibility to make our technical advancements understandable to our colleagues and the public, to promote the beneficial uses of these technologies, and to mitigate the harms of these technologies.
Known dangers include the use of machine learning systems to power networks of surveillance, produce misinformation, and the inappropriate use of training data, and we endeavor to play a role in counteracting such harms.

\bibliographystyle{alpha}
\bibliography{main_arxiv}

\newpage

\appendix

\section{Two results about $L2$-regularization}

In this appendix, we prove two results about $L2$-regularization.
The first, Theorem \ref{thm:answer_to_rho_q}, is a condition on functions such that if $L$ satisfies this condition, then under $L2$-regularization, $L + \eps R_s$ is Morse.
The condition appearing in Theorem \ref{thm:answer_to_rho_q} implies that at every critical point of the unregularized function $L$, the corank of the Hessian is at most one; this condition is not satisfied by the loss function of a typical neural net.

In the second, Proposition \ref{prop:polys_L2-reg}, we consider the question of which polynomials $L$ become Morse under generic $L2$-regularization.
We show that the set of such ``good'' polynomials is semialgebraic, and has positive codimension.
It may be possible to prove an analog of Proposition \ref{prop:polys_L2-reg} in the setting of neural nets with polynomial activation function.
We leave this task to future work.

In their current form, neither result is sufficient to prove an analogue of Theorem \ref{thm:using_lerario} for $L2$-regularization.
However these results demonstrate how techniques from differential and semialgebraic geometric can be used to probe the relationship between $L2$-regularization and the Morse property, and may be of use in other settings.
In future work, it may be possible to extend one or both to the setting of deep neural networks.

\subsection{Sum of squares}

In this subsection, we consider the following version of Open Question \ref{openquestion}.

\begin{question}
Consider a smooth function $L\colon \bR^d \to \bR$.
What conditions must $L$ satisfy so that for almost every $\eps \in \bR$, $L(x_1,\ldots,x_n) + \eps\bigl(x_1^2+\cdots+x_d^2\bigr)$ is Morse?
\end{question}

\noindent
As we saw in Example \ref{ex:counterex}, we cannot expect that a given $L$ will satisfy this condition, without imposing conditions on $L$.

We can immediately reformulate this in polar coordinates, by rewriting $L(x_1,\ldots,x_n)$ as \\ $L(r,\theta_1,\ldots,\theta_{d-1})$ for suitably-chosen radial and angular coordinates.
(This reparametrization is valid only away from a certain codimension-2 locus.)
In this reformulation, the question becomes the following:

\begin{question}
Consider a smooth function $L\colon \bR_{>0} \times \bR^{d-1} \to \bR$.
What conditions must $L$ satisfy so that for almost every $\eps \in \bR$, $L(r,\theta_1,\ldots,\theta_{d-1}) + \eps r^2$ is Morse?
\end{question}

\noindent
Finally, note that $r \mapsto r^2$ is a diffeomorphism from $\bR_{>0}$ to itself.
We can therefore replace the second version of the question with a third version, which is even simpler:

\begin{question}
\label{q:rho}
Consider a smooth function $L\colon \bR_{>0} \times \bR^{d-1} \to \bR$.
What conditions must $L$ satisfy so that for almost every $\eps \in \bR$, $L(\rho,\theta_1,\ldots,\theta_{d-1}) + \eps\rho$ is Morse?
\end{question}

\subsubsection{An answer to Question \ref{q:rho}}

In this subsubsection, we prove the following theorem, which is one answer to Question \ref{q:rho}.

\begin{theorem}
\label{thm:answer_to_rho_q}
Suppose that $L(\rho,\theta_1,\ldots,\theta_{d-1})$ is a smooth function satisfying the following condition:
\begin{align}
\label{eq:misses_S_hypothesis}
\nabla_{\ul\theta}L(p_0) = 0
\implies
\text{rank}\bigl(
\partial_\rho\nabla_{\ul\theta}L(p_0)
\:\:\:
\nabla_{\ul\theta}^2L(p_0)
\bigr)
=
d-1.
\end{align}
Then $L + \eps\rho$ is Morse for almost every $\eps \in \bR$.
Moreover, the condition \eqref{eq:misses_S_hypothesis} defines a residual subset, in both the weak and strong topologies.
\end{theorem}

\noindent
Before we prove Theorem \ref{thm:answer_to_rho_q}, we prove some supporting lemmas.
We begin by defining a certain 2-jet space:
\begin{align}
J^2(\bR_{>0}\times\bR^{d-1},\bR)
\coloneqq
\bigl\{
(\rho,\ul\theta,y,\sigma,\xi,H)
\in
\bR_{>0}
\times
\bR^{d-1}
\times
\bR
\times
\bR
\times
\bR^{d-1}
\times
\Sym(d)
\bigr\}.
\end{align}
Moreover, we make the following notation for the different parts of $H$:
\begin{align}
\left(
\begin{array}{rr}
h_1 & h_2^T
\\
h_2 & H_2
\end{array}\right).
\end{align}
A function $L\colon \bR_{>0}\times\bR^{d-1} \to \bR$ induces a map to the 2-jet space, via the following formula:
\begin{align}
j^2L
\colon
\bR_{>0}\times\bR^{d-1}
&\to
J^2(\bR_{>0}\times\bR^{d-1},\bR),
\\
(\rho,\ul\theta)
&\mapsto
\bigl(\rho,
\ul\theta,
L(\rho,\ul\theta),
\partial_\rho L(\rho,\ul\theta),
\nabla_{\ul\theta}L(\rho,\ul\theta),
\nabla^2L(\rho,\ul\theta)\bigr).
\nonumber
\end{align}
In the upcoming lemma, we will consider the question of whether it is typical for the image of $j^2L$ to intersect the following set:
\begin{align}
S
\coloneqq
\left\{
(\rho,\ul\theta,y,\sigma,\xi,H) \in J^2(\bR_{>0}\times\bR^{d-1},\bR)
\:\left|\:
\xi=0,
\rk\bigl(\begin{array}{rr}
h_2
&
H_2
\end{array}\bigr) \leq d-2
\right.\right\}
\end{align}
We note that $j^2L$ missing $S$ is, in fact, exactly the condition \eqref{eq:misses_S_hypothesis} in Theorem \ref{thm:answer_to_rho_q}.

\begin{lemma}
\label{lem:jet_misses_S}
For generic $L$, $j^2L$ misses $S$.
\end{lemma}

\begin{proof}
{\bf Step 1:}
{\it In the space $\Sym(k)$ of symmetric $k\times k$ matrices, the locus $Z$ of matrices $M$ satisfying $\rk M \leq k-2$ has codimension at least 2.}

\medskip

\noindent
Consider the locus $Z' \subset \Sym(k)$ of matrices $M$ with $\rk M \leq k-1$.
This locus is defined by a nonzero irreducible polynomial, so it is an irreducible subvariety of codimension at least 1.
We have $Z \subsetneq Z'$, and $Z$ is an algebraic subvariety of $Z'$, hence $Z$ must have codimension at least 2 in $\Sym(k)$.

\medskip

\noindent
{\bf Step 2:}
{\it In the space $\Sym(d)$ of symmetric $d\times d$ matrices, the locus $X$ of matrices $H$ satisfying
\begin{align}
\rk\bigl(\begin{array}{rr}
h_2
&
H_2
\end{array}\bigr) \leq d-2
\end{align}
has codimension at least 2.}

\medskip

\noindent
We begin by observing that the codimension in question is equal to the codimension of
\begin{align}
X
\coloneqq
\bigl\{
(b \:\: A)
\:\big|\:
A = A^T,
\:
\rk(b \:\: A) \leq d-2
\bigr\}
\end{align}
inside the space of $(d-1)\times d$ matrices $(b \:\: A)$ satisfying $A = A^T$.

Suppose that $(b \:\: A)$ lies in $X$.
There are then two possibilities:
\begin{enumerate}
\item
$\rk A \leq d-3$.

\item
$\rk A = d-2$.
\end{enumerate}
The first case is straightforward: by Step 1, the set of matrices $A$ with $\rk A \leq d-3$ has codimension at least 2 in $\Sym(d-1)$.

Dealing with the second case is slightly more involved.
If $\rk A = d-2$, we have $\rk(b \:\: A) = d-2$ if and only if $b \in \im(A) \simeq \ker(A)^\perp \simeq \bR^{d-2}$.
Define a space $Y$, which is a mild repackaging of the space of matrices we are currently considering, like so:
\begin{align}
Y
\coloneqq
\bigl\{
(b \:\: A)
\:\big|\:
A = A^T,
\:
\rk(A) = d-2,
\:
b \in \ker(A)^\perp
\bigr\}.
\end{align}
By the above reasoning, $Y$ is a vector bundle of rank $d-2$ over $\{A \in \Sym(d-1) \:|\: \rk A = d-2\}$.
We can use this to compute the dimension of $Y$:
\begin{align}
\dim Y
&=
(\dim \Sym(d-1) - 1) + (d-2)
\\
&=
\frac{d(d-1)}2 + d - 3
\nonumber
\\
&=
\left(\frac{d(d+1)}2 - 1\right) - 2.
\nonumber
\end{align}
Therefore the dimension of $Y$ has dimension 2 less than the dimension of that of the space of $(d-1)\times d$ matrices $(b \:\: A)$ satisfying $A = A^T$.

\medskip

\noindent
{\bf Step 3:}
{\it We prove the lemma.}

\medskip

\noindent
By the Jet Transversality Theorem \cite[Theorem 3.2.8 and Exercise 3.8(b)]{hirsch}, for a generic $\cC^k$-function $L$ (for $k \geq 3$ and with respect to either the weak or strong topology), its jet $j^2L$ is transverse to $S$.
It follows from Step 2 that $S$ has codimension at least $d+1$.
Since the domain of $L$ has dimension $d$, a generic $L$ must miss $S$.
\end{proof}

We are now ready to prove the main result of this subsubsection.

\begin{proof}[Proof of Theorem \ref{thm:answer_to_rho_q}]
By Lemma \ref{lem:jet_misses_S}, it suffices to show that for $L$ for which $j^2L$ misses $S$, $L + \eps\rho$ is Morse for almost every $\eps$.
Fix such an $L$.

Define $\Gamma \subset \bR_{>0}\times\bR^{d-1}$ by the formula
\begin{align}
\Gamma
\coloneqq
\left\{
\frac{\partial L}{\partial\theta_1} = 0,
\ldots,
\frac{\partial L}{\partial\theta_{d-1}} = 0
\right\}.
\end{align}
Since $j^2L$ misses $S$, the matrix $(\partial_\rho\nabla_{\ul\theta}L \:\: \nabla_{\ul\theta}^2L)$ has full rank at every point on $\Gamma$.
This matrix is exactly the Jacobian of the equations defining $\Gamma$, so $\Gamma$ is cut out transversely.
It follows that $\Gamma$ is a smooth embedded curve in $\bR_{>0} \times \bR^{d-1}$.

Consider the map
\begin{align}
\alpha
\coloneqq
\left.\frac{\partial L}{\partial\rho}\right|_\Gamma
\colon
\Gamma
\to
\bR.
\end{align}
By Sard's Theorem, almost every point in $\bR$ is a regular value of $\alpha$.
If $-\eps$ is a regular value, we claim that $(-\eps,0,\ldots,0)$ is a regular value of $\nabla L$.
Fix $p_0 \in L^{-1}\{(-\eps,0,\ldots,0)\}$; we must show that the matrix
\begin{align}
\label{eq:decomposed_hessian}
\left(\begin{array}{ll}
\partial_\rho^2L(p_0) & \bigl(\partial_\rho\nabla_{\ul\theta}L(p_0)\bigr)^T
\\
\partial_\rho\nabla_{\ul\theta}L(p_0) & \nabla_{\ul\theta}^2L(p_0)
\end{array}\right)
\end{align}
is nonsingular.
The point $p_0$ lies on $\Gamma$, so it follows from the second paragraph that the submatrix $(\partial_\rho\nabla_{\ul\theta}L(p_0) \:\: \nabla_{\ul\theta}^2L(p_0))$ has rank $d-1$.
To check that the matrix in \eqref{eq:decomposed_hessian} is nonsingular, it suffices to show that the top row $\left(\partial_\rho^2L(p_0) \:\: \bigl(\partial_\rho\nabla_{\ul\theta}L(p_0)\bigr)^T\right)$ is not orthogonal to the kernel of $(\partial_\rho\nabla_{\ul\theta}L(p_0) \hspace{0.1in} \nabla_{\ul\theta}^2L(p_0))^T$.
Fix $(\dot\rho, \dot{\ul\theta}) \in \ker(\partial_\rho\nabla_{\ul\theta}L(p_0) \:\: \nabla_{\ul\theta}^2L(p_0))^T$.
Then $(\dot\rho, \dot{\ul\theta})$ is tangent to $\Gamma$ at $p_0$, so the fact that $-\eps$ is a regular value of $\alpha$ implies that the dot product of $(\dot\rho, \dot{\ul\theta})$ and $\left(\partial_\rho^2L(p_0) \:\: \bigl(\partial_\rho\nabla_{\ul\theta}L(p_0)\bigr)^T\right)$ is nonzero.
Thus the matrix in \eqref{eq:decomposed_hessian} is nonsingular.

We have shown that for generic $-\eps \in \bR$, $(-\eps,0,\ldots,0)$ is a regular value of $\nabla L$.
The regularity of $(-\eps,0,\ldots,0)$ is equivalent to $L + \eps\rho$ being Morse.
\end{proof}

\subsection{$L2$-perturbations of polynomials}

We begin by making some definitions.
\begin{itemize}
\item
For $n \geq 1, d \geq 0$, $P_{n,d}$ is the space of polynomials in $n$ variables of degree at most $d$.

\smallskip

\item
For $L \in P_{n,d}$ and $\eps \in \bR$, we define $L_\eps$ to be the following perturbation of $L$:
\begin{align}
L_\eps(\bx)
\coloneqq
L(\bx)+\tfrac12\eps\|\bx\|^2.
\end{align}

\smallskip

\item
For $L \in P_{n,d}$, we define the set of ``bad perturbations'':
\begin{align}
B(L)
\coloneqq
\{
\eps \in \bR
\,|\,
\textrm{$L_\epsilon$ is not Morse}
\}.
\end{align}

\smallskip

\item
Denote by $M_{n,d}\subseteq P_{n,d}$ the set of polynomials $L$ such that $B(L)$ is finite.
\end{itemize}

\subsubsection{For almost every polynomial, almost every $L2$-perturbation is Morse}

Denote by $M_{n,d}\subseteq P_{n,d}$ the set of polynomials $L$ such that $B(L)$ is finite.

\begin{lemma}
For every $L\in P_{n,d}$, the set $B(L)$ is semialgebraic.
\end{lemma}

\begin{proof}
We begin by rewriting $B(L)$:
\begin{align}
B(L)
&=
\{
\epsilon \in \mathbb{R}
\,|\,
\exists\,
\bx \in \bR^n:
\nabla L(\bx)+\epsilon \bx = 0
\land
\det(\Hess(L)(x)+\epsilon I)=0
\}
\\
&=
\pi_1\bigl(\{
(\epsilon,\bx) \in \bR\times\bR^n
\,|\,
\nabla L(\bx)+\epsilon \bx = 0
\land
\det(\Hess(L)(x)+\epsilon I)=0
\}\bigr),
\nonumber
\end{align}
where $\pi_1\colon \bR\times\bR^n \to \bR$ is the projection onto the first factor.
We have written $B(L)$ as the projection of an algebraic set, hence $B(L)$ is semialgebraic.
\end{proof}

The following proposition says that for almost all polynomials $L \in P_{n,d}$, the perturbation $L_\eps$ fails to be Morse only for finitely many $\eps \in \bR$.

\begin{proposition}
\label{prop:polys_L2-reg}
For any $n \geq 1$ and $d \geq 0$, $M_{n,d}$ is semialgebraic, and its complement has positive codimension.
\end{proposition}

\begin{proof}
Our proof will involve the following ``universal moduli space of degenerate critical points'' in a key way:
\begin{align}
T_{n,d}
\coloneqq
\{
(\epsilon, \bx, L)
\,|\,
\nabla L(\bx)+\epsilon\bx=0,\, \det(\Hess(L)(\bx)+\eps I)=0
\}
\subset
\bR\times\bR^n\times P_{n,d}.
\end{align}

\noindent
{\bf Step 1:}
{\it
We prove that $\dim P_{n,d} = \dim T_{n,d}$.
}

\medskip

\noindent
We begin by rewriting $L$:
\begin{align}
L(\bx)
=
a_0 + \langle \ba_1, \bx\rangle + \langle \bx, A_2\bx\rangle + R_3(\bx)
\end{align}
for $a_0\in \bR$, $\ba_1\in \bR^n$, $A_2\in \Sym(n, \bR)$, and $R_3(\bx)$ a polynomial of degree at least 3.
In the following two bullets, we reformulate the two defining equations of $T_{n,d}$ in terms of this decomposition.
\begin{itemize}
\item
The equation
$\nabla L(\bx) + \eps\bx = 0$
becomes
\begin{align}
\ba_1 + 2A_2\bx + \nabla R_3(\bx) +\eps\bx
=
0.
\end{align}
Solving for $\ba_1$ yields the following:
\begin{align}
\ba_1
=
-2A_2\bx
-
\nabla R_3(\bx)-\eps\bx.
\end{align}

\smallskip

\item
The equation $\det(\Hess(L)(\bx) + \eps I) = 0$ becomes
\begin{align}
2A_2 + \Hess(R_3)(\bx) + \eps I
=
0.
\end{align}
Note that this equation does not involve the coefficients $a_0$ and $\ba_1$.
\end{itemize}

Using these two bullets, we rewrite $T_{n,d}$:
\begin{align}
T_{n,d}
=
\{
(\eps, \bx, L)
\,|\,
a_0 \in \bR,
\,
\ba_1=-2A_2\bx-\nabla R_3(\bx)-\eps\bx,
\,
\det\bigl(2A_2+\Hess(R_3)(\bx)+\eps I\bigr) = 0
\}.
\end{align}
The equation
\begin{align}
\det\bigl(2A_2 + \Hess(R_3)(\bx) + \eps I\bigr) = 0
\end{align}
defines an algebraic hypersurface $H$ in the space $\{(\eps, \bx, A_2, R_3)\}$.
Indeed, $H$ is algebraic by construction, and proper since  $\det(2I)\neq 0$ (we are evaluating the defining polynomial at $A_2=I$ and $R_3=0$, $\bx=0$, $\epsilon=0$).
It follows that $\codim H = 1$.
Examining the definition of $T_{n,d}$, we see that $T_{n,d}$ has codimension $n+1$ inside $\bR \times \bR^{n+1} \times P_{n,d}$.
It follows that $\dim T_{n,d} = \dim P_{n,d}$.

\medskip

\noindent
{\bf Step 2:}
{\it We prove the proposition.}

\medskip

\noindent
Define $\pi_3\colon T_{n,d}\to P_{n,d}$ to be the projection onto the third factor.
By semialgebraic triviality, there exists a finite semialgebraic decomposition
\begin{align}
P_{n,d}
=
\bigsqcup_{j=1}^r P_j
\end{align}
such that for every $k$ there is a semialgebraic fiber $F_k$ and a semialgebraic homeomorphism $\pi_3^{-1}(P_k) \simeq P_k \times F_k$.
In particular, $\dim T_{n,d} = \max_k (\dim P_k +\dim F_k)$.
It follows from this and the equality $\dim T_{n,d} = \dim P_{n,d}$ that if $P_k$ has dimension $\dim P_k =\dim P_{n,d}$, then $\dim F_k = 0$.

Define $\wt M_{n,d} \subset P_{n,d}$:
\begin{align}
\wt M_{n,d}
\coloneqq
\bigsqcup_{k:\: \dim P_k = \dim P_{n,d}}
P_k.
\end{align}
Then $M_{n,d} = \wt M_{n,d}$, since a semialgebraic subset of $\bR$ has finite cardinality if and only if it has dimension 0.
$\wt M_{n,d}$ is a semialgebraic set whose complement has positive codimension, so we have proven the proposition.
\end{proof}

\end{document}